\documentclass{article}

\usepackage[english]{babel}
\usepackage[utf8]{inputenc}

\usepackage{amsmath}
\usepackage{amssymb}

\usepackage{hyperref}
\usepackage{enumerate}
\usepackage{bbm} % for \mathbbm{1}
\usepackage{tikz}
\usepackage{longtable} % for pagebreaks in the table

\usepackage{natbib}
\bibpunct{(}{)}{;}{a}{,}{,}

\usepackage{mythm}
\usepackage{aixi}

\DeclareMathOperator*{\xor}{xor} % elementwise binary exclusive or

\AtBeginDocument{

}
\newcommand{\appendixref}[1]{{\hyperref[#1]{Appendix~\ref*{#1}}}}

\title{Bad Universal Priors and Notions of Optimality}
\author{Jan Leike and Marcus Hutter}

\begin{document}

\maketitle

\begin{abstract}
A big open question of algorithmic information theory is
the choice of the universal Turing machine (UTM).
For Kolmogorov complexity and Solomonoff induction we have invariance theorems:
the choice of the UTM changes bounds only by a constant.
For the universally intelligent agent AIXI~\citep{Hutter:2005}
no invariance theorem is known.
Our results are entirely negative:
we discuss cases
in which unlucky or adversarial choices of the UTM
cause AIXI to misbehave drastically.
We show that
Legg-Hutter intelligence and thus balanced Pareto optimality
is entirely subjective, and that
every policy is Pareto optimal
in the class of all computable environments.
This undermines all existing optimality properties for AIXI.
While it may still serve as a gold standard for AI,
our results imply that
AIXI is a \emph{relative} theory,
dependent on the choice of the UTM.
\end{abstract}


%%%%%%%%%%%%%%%%%%%%%%%%%%%%%%%%%%%%%%%%%%%%%%%%%%%%%%%%%%%%%%%%%%%%
\section{Introduction}
\label{sec:introduction}

%\paragraph{UTMs and Invariance Theorems}
The choice of the universal Turing machine (UTM) has been a big open question
in algorithmic information theory for a long time.
While attempts have been made~\citep{Mueller:2010}
no answer is in sight.
The \emph{Kolmogorov complexity} of a string,
the length of the shortest program that prints this string,
depends on this choice.
However, there are \emph{invariance theorems}~\cite[Thm.\ 2.1.1 \& Thm.\ 3.1.1]{LV:2008}
which state that changing the UTM
changes Kolmogorov complexity only by a constant.
When using the \emph{universal prior} $M$
introduced by \cite{Solomonoff:1964,Solomonoff:1978}
to predict any deterministic computable binary sequence,
the number of wrong predictions is bounded by (a multiple of)
the Kolmogorov complexity of the sequence~\citep{Hutter:2001error}.
Due to the invariance theorem,
changing the UTM
changes the number of errors only by a constant.
In this sense, compression and prediction work for any choice of UTM.

%\paragraph{AIXI.}
\cite{Hutter:2000,Hutter:2005} defines the universally intelligent agent AIXI,
which is targeted at
the \emph{general reinforcement learning problem}~\citep{SB:1998}.
It extends Solomonoff induction to the interactive setting.
AIXI is a Bayesian agent,
using a universal prior on the set of all computable environments;
actions are taken according to
the maximization of expected future discounted rewards.
Closely related is the intelligence measure defined by \cite{LH:2007int},
a mathematical performance measure for general reinforcement learning agents:
defined as
the discounted rewards achieved across all computable environments,
weighted by the universal prior.

%\paragraph{Positive Results}
There are several known positive results about AIXI.
It has been proven to be
\emph{Pareto optimal}~\cite[Thm.\ 2 \& Thm.\ 6]{Hutter:2002},
\emph{balanced Pareto optimal}~\cite[Thm.\ 3]{Hutter:2002}, and
has maximal Legg-Hutter intelligence.
Furthermore,
AIXI asymptotically learns to predict the environment perfectly and with
a small total number of errors analogously to
Solomonoff induction~\cite[Thm.\ 5.36]{Hutter:2005},
but only \emph{on policy}:
AIXI learns to correctly predict the value (expected future rewards)
of its own actions,
but generally not the value of counterfactual actions that it does not take.

%\paragraph{Our Questions}
\cite{Orseau:2010,Orseau:2013} showed that
AIXI does not achieve asymptotic optimality
in all computable environments.
So instead, we may ask the following weaker questions.
Does AIXI succeed in every partially observable Markov decision process (POMDP)/%
(ergodic) Markov decision process (MDP)/%
bandit problem/%
sequence prediction task?
In this paper we show that without further assumptions on the UTM,
we cannot answer any of the preceding questions in the affirmative.
More generally, there can be no invariance theorem for AIXI.
As a reinforcement learning agent,
AIXI has to balance between exploration and exploitation.
Acting according to any (universal) prior
does not lead to enough exploration, and
the bias of AIXI's prior is retained indefinitely.
For bad priors this can cause serious malfunctions.
However, this problem can be alleviated
by adding an extra exploration component to AIXI~\cite[Ch.\ 5]{Lattimore:2013},
similar to knowledge-seeking agents~\citep{Orseau:2014ksa,OLH:2013ksa},
or by the use of optimism~\citep{SH:12optopt}.

%\paragraph{Results}
In \autoref{sec:bad-universal-priors} we give two examples of universal priors
that cause AIXI to misbehave drastically.
In case of a finite lifetime,
the \emph{indifference prior} makes all actions equally preferable to AIXI
(\autoref{ssec:indifference-prior}).
Furthermore, for any computable policy $\pi$
the \emph{dogmatic prior}
makes AIXI stick to the policy $\pi$ as long as
expected future rewards do not fall too close to zero
(\autoref{ssec:dogmatic-prior}).
%
%\paragraphs{Implications}
This has profound implications.
We show in \autoref{sec:intelligence}
that if we measure Legg-Hutter intelligence with respect to
a \emph{different} universal prior,
AIXI scores arbitrarily close to the minimal intelligence while
any computable policy can score arbitrarily close to the maximal intelligence.
This makes the Legg-Hutter intelligence score
and thus balanced Pareto optimality
relative to the choice of the UTM.
Moreover,
in \autoref{sec:Pareto-optimality} we show that
in the class of all computable environments,
\emph{every} policy is Pareto optimal.
This undermines all existing optimality results for AIXI.
We discuss the implications of these results for
the quest for a \emph{natural} universal Turing machine and
optimality notions of general reinforcement learners
in \autoref{sec:discussion}.
A list of notation is provided in \appendixref{app:notation}.

%%%%%%%%%%%%%%%%%%%%%%%%%%%%%%%%%%%%%%%%%%%%%%%%%%%%%%%%%%%%%%%%%%%%
\section{Preliminaries and Notation}
\label{sec:preliminaries}

%\paragraph{Strings.}
The set $\X^* := \bigcup_{n=0}^\infty \X^n$ is
the set of all finite strings over the alphabet $\X$,
the set $\X^\infty$ is
the set of all infinite strings
over the alphabet $\X$, and
the set $\X^\sharp := \X^* \cup \X^\infty$ is their union.
The empty string is denoted by $\epsilon$, not to be confused
with the small positive real number $\varepsilon$.
Given a string $x \in \X^\sharp$, we denote its length by $|x|$.
For a (finite or infinite) string $x$ of length $\geq k$,
we denote with $x_{1:k}$ the first $k$ characters of $x$,
and with $x_{<k}$ the first $k - 1$ characters of $x$.
The notation $x_{1:\infty}$ stresses that $x$ is an infinite string.
We write $x \sqsubseteq y$ iff $x$ is a prefix of $y$, i.e.,
$x = y_{1:|x|}$.

%\paragraph{Reinforcement Learning.}
In reinforcement learning,
the agent interacts with an environment in cycles:
at time step $t$ the agent chooses an \emph{action} $a_t \in \A$ and
receives a \emph{percept} $e_t = (o_t, r_t) \in \E$
consisting of an \emph{observation} $o_t \in \O$
and a real-valued \emph{reward} $r_t \in \mathbb{R}$;
the cycle then repeats for $t + 1$.
A \emph{history} is an element of $\H$.
We use $\ae \in \A \times \E$ to denote one interaction cycle,
and $\ae_{<t}$ to denote a history of length $t - 1$.
The goal in reinforcement learning is
to maximize total discounted rewards.
A \emph{policy} is a function $\pi: \H \to \A$
mapping each history to the action taken after seeing this history.
A history $\ae_{<t}$ is \emph{consistent with policy $\pi$} iff
$\pi(\ae_{<k}) = a_k$ for all $k < t$.

%\paragraph{Lower Semicomputable.}
A function $f: \X^* \to \mathbb{R}$ is \emph{lower semicomputable} iff
the set $\{ (x, q) \in \X^* \times \mathbb{Q} \mid f(x) > q \}$
is recursively enumerable.
%
%\paragraph{Conditional Semimeasures.}
A \emph{conditional semimeasure} $\nu$ is
a probability measure over finite and infinite strings of percepts
given actions as input where
$\nu(e_{<t} \dmid a_{1:\infty})$ denotes
the probability of receiving percepts $e_{<t}$
when taking actions $a_{1:\infty}$.
Formally, $\nu$ maps $\A^\infty$ to
a probability distribution over $\E^\sharp$.
Thus the environment might assign positive probability to
finite percept sequences.
One possible interpretation for this is that
there is a non-zero chance that the environment ends:
it simply does not produce a new percept.
Another possible interpretation is that
there is a non-zero chance of death for the agent.
However, nothing hinges on the interpretation;
the use of (unnormalized) semimeasures
is primarily a technical trick.

The conditional semimeasure $\nu$ is \emph{chronological} iff
the first $t - 1$ percepts are independent of
future actions $a_k$ for $k \geq t$, i.e.,
$\nu(e_{<t} \dmid a_{1:k}) = \nu(e_{<t} \dmid a_{<t})$.
Despite their name,
conditional semimeasures do not denote a conditional probability;
$\nu$ is not a joint probability distribution over actions and percepts.
%
%\paragraph{Environments.}
We model environments as
lower semicomputable chronological conditional semimeasures (LSCCCS)%
~\cite[Sec.\ 5.1.1]{Hutter:2005};
the class of all such environments is denoted as $\Mlscccs$.
We also use the larger set
of all chronological conditional semimeasures $\Mccs$.

%\paragraph{The universal prior}
A \emph{universal prior} is
a function $w: \Mlscccs \to [0, 1]$ such that
$w_\nu := w(\nu) > 0$ for all $\nu \in \Mlscccs$ and
$\sum_{\nu \in \Mlscccs} w_\nu \leq 1$.
A universal prior $w$ gives rise to a \emph{universal mixture},
\begin{equation}\label{eq:xi-mixture}
     \xi(e_{<t} \dmid a_{<t})
~:=~ \sum_{\nu \in \Mlscccs} w_\nu \nu(e_{<t} \dmid a_{<t}).
\end{equation}
If the universal prior is lower semicomputable,
then the universal mixture $\xi$ is an LSCCCS,
i.e., $\xi \in \Mlscccs$.
From a given universal monotone Turing machine $U$~\cite[Sec.\ 4.5.2]{LV:2008}
we can get a universal mixture $\xi$ in two ways.
First, we can use \eqref{eq:xi-mixture} with the prior given by
$w_\nu := 2^{-K(\nu)}$, where $K(\nu)$ is
the \emph{Kolmogorov complexity} of $\nu$'s index
in the enumeration of all LSCCCSs~\cite[Eq.\ 4.11]{LV:2008}.
Second, we can define it as the probability that
the universal monotone Turing machine $U$ generates $e_{<t}$
when fed with $a_{<t}$ and uniformly random bits:
\begin{equation}\label{eq:xi-enumeration}
     \xi(e_{<t} \dmid a_{<t})
~:=~ \sum_{p:\, e_{<t} \sqsubseteq U(p, a_{<t})} 2^{-|p|}
\end{equation}
Both definitions are equivalent,
but not necessarily equal~\cite[Lem.\ 10 \& Lem.\ 13]{WSH:2011}.

%-------------------------------%
\begin{lemma}[Mixing Mixtures]
\label{lem:mixing-mixtures}
%-------------------------------%
Let $q, q' \in \mathbb{Q}$ such that
$q > 0$, $q' \geq 0$, and $q + q' \leq 1$.
Let $w$ be any lower semicomputable universal prior,
let $\xi$ be the universal mixture for $w$, and
let $\rho$ be an LSCCCS.
Then $\xi' := q\xi + q'\rho$ is an LSCCCS and a universal mixture.
\end{lemma}
\begin{proof}
$\xi'$ is given by the universal prior $w'$ with $w' := qw + q'\mathbbm{1}_\rho$.
\end{proof}

Throughout this paper, we make the following assumptions.

\begin{assumption}\label{ass:aixi}
\begin{enumerate}[(a)]
\item \label{ass:bounded-rewards}
	Rewards are bounded between $0$ and $1$.
\item \label{ass:finite-actions-and-percepts}
	The set of actions $\A$ and the set of percepts $\E$
	are both finite.
\end{enumerate}
\end{assumption}

%\paragraph{Discounting.}
We fix a \emph{discount function}
$\gamma: \mathbb{N} \to \mathbb{R}$ with
$\gamma_t := \gamma(t) \geq 0$ and $\sum_{t=1}^\infty \gamma_t < \infty$.
The \emph{discount normalization factor} is defined as
$\Gamma_t := \sum_{i=t}^\infty \gamma_i$.
There is no requirement that $\gamma_t > 0$ or $\Gamma_t > 0$.
If $m := \min \{ t \mid \Gamma_{t+1} = 0 \}$ exists,
we say the agent has a \emph{finite lifetime $m$} and
does not care what happens afterwards.

%-------------------------------%
\begin{definition}[Value Function]
\label{def:discounted-value}
%-------------------------------%
The \emph{value} of a policy $\pi$ in an environment $\nu$
given history $\ae_{<t}$ is defined as
$V^\pi_\nu(\ae_{<t}) := V^\pi_\nu(\ae_{<t} \pi(\ae_{<t}))$ and
\begin{align*}
%      V^\pi_\nu(\ae_{<t})
%~&:=~ V^\pi_\nu(\ae_{<t} \pi(\ae_{<t})), \\
      V^\pi_\nu(\ae_{<t} a_t)
~&:=~ \frac{1}{\Gamma_t} \sum_{e_t \in \E}
        \big( \gamma_t r_t + \Gamma_{t+1} V^\pi_\nu(\ae_{1:t}) \big)
        \nu(e_{1:t} \mid e_{<t} \dmid a_{1:t})
\end{align*}
if $\Gamma_t > 0$ and $V^\pi_\nu(\ae_{<t}) := 0$ if $\Gamma_t = 0$.
The \emph{optimal value} is defined as
$V^*_\nu(h) := \sup_\pi V^\pi_\nu(h)$.
\end{definition}

%-------------------------------%
\begin{definition}[{Optimal Policy~\cite[Def.\ 5.19 \& 5.30]{Hutter:2005}}]
\label{def:optimal-policy}
%-------------------------------%
A policy $\pi$ is \emph{optimal in environment $\nu$ ($\nu$-optimal)} iff
for all histories
$\pi$ attains the optimal value:
$V^\pi_\nu(h) = V^*_\nu(h)$
for all $h \in \H$.
The action $\pi(h)$ is an \emph{optimal action} iff
$\pi(h) = \pi^*_\nu(h)$ for some $\nu$-optimal policy $\pi^*_\nu$.
\end{definition}

Formally,
AIXI is defined as a policy $\pi^*_\xi$
that is optimal in the universal mixture $\xi$.
Since there can be more than one $\xi$-optimal policy,
this definition is not unique.
If there two optimal actions $\alpha \neq \beta \in \A$,
we call it an \emph{argmax tie}.
Which action we take in case of a tie (how we break the tie)
is irrelevant and can be arbitrary.
We assumed that
the discount function is summable,
rewards are bounded
(\hyperref[ass:bounded-rewards]%
{Assumption~\ref*{ass:aixi}\ref*{ass:bounded-rewards}}), and
actions and percepts spaces are both finite
(\hyperref[ass:finite-actions-and-percepts]%
{Assumption~\ref*{ass:aixi}\ref*{ass:finite-actions-and-percepts}}).
Therefore an optimal policy exists
for every environment $\nu \in \Mlscccs$~\cite[Thm.\ 10]{LH:2014discounting},
in particular for any universal mixture $\xi$.

%-------------------------------%
\begin{lemma}[{Discounted Values~\cite[Lem.\ 2.5]{Lattimore:2013}}]
\label{lem:discounted-values}
%-------------------------------%
If two policies $\pi_1$ and $\pi_2$ coincide for the first $k$ steps
($\pi_1(\ae_{<t}) = \pi_2(\ae_{<t})$
for all histories $\ae_{<t}$ consistent with $\pi_1$ and $t \leq k$),
then
\[
       \big| V^{\pi_1}_\nu(\epsilon) - V^{\pi_2}_\nu(\epsilon) \big|
~\leq~ \frac{\Gamma_{k+1}}{\Gamma_1}
\text{ for all environments $\nu \in \Mccs$}.
\]
\end{lemma}
\begin{proof}
Since the policies $\pi_1$ and $\pi_2$ coincide for the first $k$ steps,
they produce the same expected rewards for the first $k$ steps.
Therefore
\begin{align*}
        \big| V^{\pi_1}_\nu(\epsilon) - V^{\pi_2}_\nu(\epsilon) \big|
~&=~    \left| \sum_{e_{1:k}} \frac{\Gamma_{k+1}}{\Gamma_1} \big(
          V^{\pi_1}_\nu(\ae_{1:k}) - V^{\pi_2}_\nu(\ae_{1:k})
        \big) \nu(e_{1:k} \dmid a_{1:k}) \right| \\
~&\leq~ \sum_{e_{1:k}} \frac{\Gamma_{k+1}}{\Gamma_1} \big|
          V^{\pi_1}_\nu(\ae_{1:k}) - V^{\pi_2}_\nu(\ae_{1:k})
        \big| \nu(e_{1:k} \dmid a_{1:k})
~\leq~  \frac{\Gamma_{k+1}}{\Gamma_1},
\end{align*}
where $a_t := \pi_1(\ae_{<t}) = \pi_2(\ae_{<t})$ for all $t \leq k$.
The last inequality follows
since $\nu$ is a semimeasure, $0 \leq V^\pi_\nu \leq 1$ and hence
$|V^{\pi_1}_\nu(\ae_{1:k}) - V^{\pi_2}_\nu(\ae_{1:k})| \leq 1$.
\end{proof}

%%%%%%%%%%%%%%%%%%%%%%%%%%%%%%%%%%%%%%%%%%%%%%%%%%%%%%%%%%%%%%%%%%%%
\section{Bad Universal Priors}
\label{sec:bad-universal-priors}

%%%%%%%%%%%%%%%%%%%%%%%%%%%%%%%%%%%%%%%%%%%%%%%%%%%%%%%%%%%%%%%%%%%%
\subsection{The Indifference Prior}
\label{ssec:indifference-prior}

In this section
we consider AIXI with a finite lifetime $m$, i.e., $\Gamma_{m+1} = 0$.
The following theorem constructs the \emph{indifference prior},
a universal prior $\xi'$ that
causes argmax ties for the first $m$ steps.
Since we use a discount function that only cares about the first $m$ steps,
all policies are $\xi'$-optimal policies.
Thus AIXI's behavior only depends on how we break argmax ties.

%-------------------------------%
\begin{theorem}[Indifference Prior]
\label{thm:indifference-prior}
%-------------------------------%
If there is an $m$ such that $\Gamma_{m+1} = 0$,
then there is a universal mixture $\xi'$ such that
all policies are $\xi'$-optimal.
\end{theorem}
\begin{proof}
First, we assume that the action space is binary, $\A = \{ 0, 1 \}$.
Let $U$ be the reference UTM and
define the UTM $U'$ by
\[
     U'(s_{1:m}p, a_{1:t})
~:=~ U(p, a_{1:t} \xor s_{1:t}),
\]
where $s_{1:m}$ is a binary string of length $m$ and $s_k := 0$ for $k > m$.
($U'$ has no programs of length $\leq m$.)
Let $\xi'$ be the universal mixture given by $U'$
according to \eqref{eq:xi-enumeration}.
\begingroup
\allowdisplaybreaks
\begin{align*}
     \xi'(e_{1:m} \dmid a_{1:m})
~&=~ \sum_{p:\, e_{1:m} \sqsubseteq U'(p, a_{1:m})} 2^{-|p|} \\
~&=~ \sum_{s_{1:m}p':\, e_{1:m} \sqsubseteq U'(s_{1:m}p', a_{1:m})} 2^{-m-|p'|} \\
~&=~ \sum_{s_{1:m}}~ \sum_{p':\, e_{1:m} \sqsubseteq U(p', a_{1:m} \xor s_{1:m})} 2^{-m-|p'|} \\
~&=~ \sum_{s_{1:m}}~ \sum_{p':\, e_{1:m} \sqsubseteq U(p', s_{1:m})} 2^{-m-|p'|},
\end{align*}
\endgroup
which is independent of $a_{1:m}$.
Hence the first $m$ percepts are independent of the first $m$ actions.
But the percepts' rewards after time step $m$ do not matter
since $\Gamma_{m+1} = 0$ (\autoref{lem:discounted-values}).
Because the environment is chronological,
the value function must be independent of all actions.
Thus every policy is $\xi'$-optimal.

For finite action spaces $\A$ with more than $2$ elements,
the proof works analogously
by making $\A$ a cyclic group and using the group operation instead of $\xor$.
\end{proof}

The choice of $U'$ in the proof of \autoref{thm:indifference-prior}
is \emph{unnatural} since its shortest program has length greater than $m$.
Moreover, the choice of $U'$ depends on $m$.
If we increase AIXI's lifetime
while fixing the UTM $U'$,
\autoref{thm:indifference-prior} no longer holds.
For Solomonoff induction, there is an analogous problem:
when using Solomonoff's prior $M$ to predict a deterministic binary sequence $x$,
we make at most $K(x)$ errors.
In case the shortest program has length $> m$,
there is no guarantee that we make less than $m$ errors.

%%%%%%%%%%%%%%%%%%%%%%%%%%%%%%%%%%%%%%%%%%%%%%%%%%%%%%%%%%%%%%%%%%%%
\subsection{The Dogmatic Prior}
\label{ssec:dogmatic-prior}

In this section
we define a universal prior
that assigns very high probability of going to hell (reward $0$ forever)
if we deviate from a given computable policy $\pi$.
For a Bayesian agent like AIXI,
it is thus only worth deviating from the policy $\pi$
if the agent thinks that the prospects of following $\pi$ are very poor already.
We call this prior the \emph{dogmatic prior},
because the fear of going to hell makes AIXI conform to
any arbitrary `dogmatic ideology' $\pi$.
AIXI will only break out
if it expects $\pi$ to give very low future payoff;
in that case the agent does not have much to lose.

%-------------------------------%
\begin{theorem}[Dogmatic Prior]
\label{thm:dogmatic-prior}
%-------------------------------%
Let $\pi$ be any computable policy,
let $\xi$ be any universal mixture, and
let $\varepsilon > 0$.
There is a universal mixture $\xi'$ such that
for any history $h$ consistent with $\pi$ and $V^\pi_\xi(h) > \varepsilon$,
the action $\pi(h)$ is the unique $\xi'$-optimal action.
\end{theorem}

The proof proceeds by constructing a universal mixture that
assigns disproportionally high probability to
an environment $\nu$ that
sends any policy deviating from $\pi$ to hell.
Importantly, the environment $\nu$ produces observations according to
the universal mixture $\xi$.
Therefore $\nu$ is indistinguishable from $\xi$ on the policy $\pi$,
so the posterior belief in $\nu$ is equal to
the prior belief in $\nu$.

\begin{proof}
We assume $(0, 0) \in \E$.
Let $\pi$ be any computable policy and define
\[
\nu(e_{1:t} \dmid a_{1:t}) :=
\begin{cases}
\xi(e_{1:t} \dmid a_{1:t}), &\text{if }
  a_k = \pi(\ae_{<k}) \;\forall k \leq t, \\
\xi(e_{<k}  \dmid a_{<k}),  &\text{if }
  k := \min \{ i \mid a_i \neq \pi(\ae_{<i}) \} \text{ exists} \\
  &\text{ and } e_i = (0, 0)\; \forall i \in \{ k, \ldots, t \}, \\
0, &\text{otherwise}.
\end{cases}
\]
The environment $\nu$ mimics the universal environment $\xi$
until it receives an action that the policy $\pi$ would not take.
From then on, it provides rewards $0$.
Since $\xi$ is a LSCCCS and $\pi$ is a computable policy,
we have that $\nu \in \Mlscccs$.

Without loss of generality we assume that $\varepsilon$ is computable,
otherwise we make it slightly smaller.
Thus
$\xi' := \tfrac{1}{2} \nu + \tfrac{\varepsilon}{2} \xi$
is a universal mixture according to \autoref{lem:mixing-mixtures}.

Let $h \in \H$ be any history consistent with $\pi$ such that
$V^\pi_\xi(h) > \varepsilon$.
In the following,
we use the shorthand notation $\rho(h) := \rho(e_{1:t} \dmid a_{1:t})$
for a conditional semimeasure $\rho$ and $h =: \ae_{1:t}$.
Since $\nu$ gives observations and rewards according to $\xi$,
we have $\nu(h) = \xi(h)$,
and thus the posterior weight $w_\nu(h)$ of $\nu$ in $V_{\xi'}^\pi(h)$ is
constant while following $\pi$:
\[
   \frac{w_\nu(h)}{w_\nu}
:= \frac{\nu(h)}{\xi'(h)}
=  \frac{\xi(h)}{\xi'(h)}
=  \frac{\xi(h)}{\frac{1}{2}\nu(h) + \frac{\varepsilon}{2} \xi(h)}
=  \frac{\xi(h)}{\frac{1}{2}\xi(h) + \frac{\varepsilon}{2} \xi(h)}
=  \frac{2}{1 + \varepsilon}
\]
Therefore linearity of $V^{\tilde\pi}_\nu$ in $\nu$%
~\cite[Thm.\ 5.31, proved in \appendixref{app:additional-material}]{Hutter:2005}
implies that for all $a \in \A$,
\begin{equation}\label{eq:linearity}
    V^\pi_{\xi'}(ha)
~=~ w_\nu(h) V^\pi_\nu(ha) + w_\xi(h) V^\pi_\xi(ha)
~=~ \tfrac{1}{1+\varepsilon} V^\pi_\nu(ha)
    + \tfrac{\varepsilon}{1+\varepsilon} V^\pi_\xi(ha).
\end{equation}
Let $\alpha := \pi(h)$ be the next action according to $\pi$, and
let $\beta \neq \alpha$ be any other action.
We have that $V^\pi_\nu = V^\pi_\xi$ by definition of $\nu$,
therefore
\begin{equation}\label{eq:Vpixi'}
    V^\pi_{\xi'}(h\alpha)
~\stackrel{\text{\eqref{eq:linearity}}}{=}~
    \tfrac{1}{1 + \varepsilon} V^\pi_\nu(h\alpha)
    + \tfrac{\varepsilon}{1 + \varepsilon} V^\pi_\xi(h\alpha)
~=~ \tfrac{1}{1 + \varepsilon} V^\pi_\xi(h\alpha)
    + \tfrac{\varepsilon}{1 + \varepsilon} V^\pi_\xi(h\alpha)
~=~ V^\pi_\xi(h\alpha)
\end{equation}
We get that $V^*_{\xi'}(h\alpha) > V^*_{\xi'}(h\beta)$:
\begin{align*}
        V^*_{\xi'}(h\alpha)
~&\geq~ V^\pi_{\xi'}(h\alpha)
~\stackrel{\text{\eqref{eq:Vpixi'}}}{=}~
        V^\pi_\xi(h\alpha)
~=~     V^\pi_\xi(h)
~>~     \varepsilon, \\
        V^*_{\xi'}(h\beta)
~&\stackrel{\text{\eqref{eq:linearity}}}{=}~
        \tfrac{1}{1 + \varepsilon} V^{\pi^*_{\xi'}}_\nu(h\beta)
        + \tfrac{\varepsilon}{1 + \varepsilon} V^{\pi^*_{\xi'}}_\xi(h\beta)
~=~     \tfrac{\varepsilon}{1 + \varepsilon} V^{\pi^*_{\xi'}}_\xi(h\beta)
~\leq~  \tfrac{\varepsilon}{1 + \varepsilon}
~<~     \varepsilon,
\end{align*}
Hence the action $\alpha$ taken by $\pi$ is
the only $\xi'$-optimal action for the history $h$.
\end{proof}

%-------------------------------%
\begin{corollary}[AIXI Emulating Computable Policies]
\label{cor:emulation}
%-------------------------------%
Let $\varepsilon > 0$ and
let $\pi$ be any computable policy.
There is a universal mixture $\xi'$ such that
for any $\xi'$-optimal policy $\pi^*_{\xi'}$ and
for any (not necessarily computable) environment $\nu \in \Mccs$,
\[
    \left| V^{\pi^*_{\xi'}}_\nu(\epsilon) - V^\pi_\nu(\epsilon) \right|
~<~ \varepsilon.
\]
\end{corollary}
\begin{proof}
Let $\varepsilon > 0$.
Since $\Gamma_k \to 0$ as $k \to \infty$,
we can choose $k$ large enough such that $\Gamma_{k+1}/\Gamma_1 < \varepsilon$.
Let $\varepsilon' > 0$ be small enough such that
$V^\pi_\xi(h) > \varepsilon'$ for all $h$ with $|h| \leq k$.
This is possible since $V^\pi_\xi(h) > 0$ for all $h$ and
the set of histories of length $\leq k$ is finite
because of \hyperref[ass:finite-actions-and-percepts]%
{Assumption~\ref*{ass:aixi}\ref*{ass:finite-actions-and-percepts}}.
We use the dogmatic prior from \autoref{thm:dogmatic-prior}
to construct a universal mixture $\xi'$
for the policy $\pi$ and $\varepsilon' > 0$.
Thus for any history $h \in \H$ consistent with $\pi$ and $|h| \leq k$,
the action $\pi(h)$ is the only $\xi'$-optimal action.
The claim now follows from \autoref{lem:discounted-values}.
\end{proof}

%-------------------------------%
\begin{corollary}[With Finite Lifetime Every Policy is an AIXI]
\label{cor:finite-lifetime-every-policy-is-AIXI}
%-------------------------------%
If $\Gamma_{m+1} = 0$ for some $m \in \mathbb{N}$,
then for any policy $\pi$
there is a universal mixture $\xi'$ such that
$\pi(h)$ is the only $\xi'$-optimal action for all histories $h$
consistent with $\pi$ and $|h| \leq m$.
\end{corollary}
In contrast to \autoref{thm:indifference-prior}
where every policy is $\xi'$-optimal for a fixed universal mixture $\xi'$,
\autoref{cor:finite-lifetime-every-policy-is-AIXI} gives
a different universal mixture $\xi'$
for every policy $\pi$ such that
$\pi$ is the only $\xi'$-optimal policy.

\begin{proof}
Analogously to the proof of \autoref{cor:emulation},
let $\varepsilon' > 0$ be small enough such that
$V^\pi_\xi(h) > \varepsilon'$ for all $h$ with $|h| \leq m$.
Again, we use the dogmatic prior from \autoref{thm:dogmatic-prior}
to construct a universal mixture $\xi'$
for the policy $\pi$ and $\varepsilon' > 0$.
Thus for any history $h \in \H$ consistent with $\pi$ and $|h| \leq m$,
the action $\pi(h)$ is the only $\xi'$-optimal action.
\end{proof}

%%%%%%%%%%%%%%%%%%%%%%%%%%%%%%%%%%%%%%%%%%%%%%%%%%%%%%%%%%%%%%%%%%%%
\section{Consequences for Legg-Hutter Intelligence}
\label{sec:intelligence}

%\paragraph{What is intelligence?}
The aim of the Legg-Hutter intelligence measure is to formalize the intuitive
notion of intelligence mathematically.
If we take intelligence to mean
\emph{an agent's ability to achieve goals in a wide range of environments}
\citep{LH:2007int},
and we weigh environments according to the universal prior,
then the intelligence of a policy $\pi$ corresponds to
the value that $\pi$ achieves in the corresponding universal mixture.
We use the results form the previous section
to illustrate some problems with this intelligence measure
in the absence of a \emph{natural} UTM.

%-------------------------------%
\begin{definition}[{Legg-Hutter Intelligence \citep{LH:2007int}})]
\label{def:intelligence}
%-------------------------------%
The \emph{intelligence}\footnote{
\cite{LH:2007int} consider a subclass of $\Mlscccs$,
the class of computable \emph{measures},
and do not use discounting explicitly.}
of a policy $\pi$ is defined as
\[
     \Upsilon_\xi(\pi)
~:=~ \sum_{\nu \in \Mlscccs} w_\nu V^\pi_\nu(\epsilon)
~ =~ V^\pi_\xi(\epsilon).
\]
\end{definition}
Typically, the index $\xi$ is omitted when writing $\Upsilon$.
However, in this paper we consider
the intelligence measure with respect to different universal mixtures,
therefore we make this dependency explicit.

%\paragraph{Intelligence scores.}
Because the value function is scaled to be in the interval $[0, 1]$,
intelligence is a real number between $0$ and $1$.
Legg-Hutter intelligence is linked to \emph{balanced Pareto optimality}:
a policy is said to be \emph{balanced Pareto optimal} iff
it scores the highest intelligence score:
\[
     \overline\Upsilon_\xi
~:=~ \sup_\pi \Upsilon_\xi(\pi)
~ =~ \Upsilon_\xi(\pi^*_\xi).
\]
AIXI is balanced Pareto optimal~\cite[Thm.\ 5.24]{Hutter:2005}.
It is just as hard to score very high on the Legg-Hutter intelligence measure
as it is to score very low:
we can always turn a reward minimizer into a reward maximizer by inverting
the rewards $r_t' := 1 - r_t$.
Hence the lowest possible intelligence score is achieved by
AIXI's twin sister, a $\xi$-expected reward minimizer:
\[
     \underline\Upsilon_\xi
~:=~ \inf_\pi \Upsilon_\xi(\pi).
\]
The heaven environment (reward $1$ forever) and
the hell environment (reward $0$ forever) are computable
and thus in the environment class $\Mlscccs$;
therefore it is impossible
to get a reward $0$ or reward $1$ in every environment.
Consequently, for all policies $\pi$,
\begin{equation}\label{eq:bounds-on-intelligence}
       0
~<   ~ \underline\Upsilon_\xi
~\leq~ \Upsilon_\xi(\pi)
~\leq~ \overline\Upsilon_\xi
~<   ~ 1.
\end{equation}
See \autoref{fig:intelligence}.
It is natural to fix the policy \texttt{random}
that takes actions uniformly at random
to have an intelligence score of $1 / 2$
by choosing a `symmetric' universal prior~\citep{LV:2013int}.

%-------------------------------%
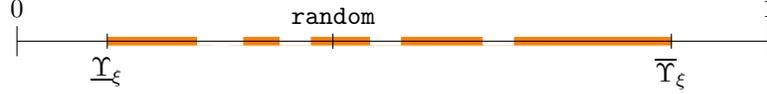
\begin{figure}[t]
%-------------------------------%
\begin{center}
\begin{tikzpicture}
\filldraw[orange] (1.2, .05) -- (8.7, .05) -- (8.7, -.05) -- (1.2, -.05);
\foreach \x in {2.4, 2.6, 3.5, 4.7, 6.2} {
  \filldraw[white] (\x, .05) -- ({\x+.4}, .05) -- ({\x+.4}, -.05) -- (\x, -.05);
}
\draw (0,0) to (10, 0);
\draw (0, -.2) to (0, .2) node[above] {$0$};
\draw (10, -.2) to (10, .2) node[above] {$1$};
\draw (4.2, -.1) to (4.2, .1) node[above] {\texttt{random}};
\draw (8.7, .1) to (8.7, -.1) node[below] {$\overline\Upsilon_\xi$};
\draw (1.2, .1) to (1.2, -.1) node[below] {$\underline\Upsilon_\xi$};
\end{tikzpicture}
\end{center}
\caption{
The Legg-Hutter intelligence measure assigns values
within the closed interval $[\underline\Upsilon_\xi, \overline\Upsilon_\xi]$;
the assigned values are depicted in orange.
By \autoref{thm:computable-policies-are-dense},
computable policies are dense in this orange set.
}
\label{fig:intelligence}
\end{figure}

AIXI is not computable~\cite[Thm.\ 14]{LH:2015computability},
hence there is no computable policy $\pi$ such that
$\Upsilon_\xi(\pi) = \underline\Upsilon_\xi$ or
$\Upsilon_\xi(\pi) = \overline\Upsilon_\xi$ for any universal mixture $\xi$.
But the next theorem tells us that
computable policies can come arbitrarily close.
This is no surprise:
by \autoref{lem:discounted-values}
we can do well on a Legg-Hutter intelligence test
simply by memorizing what AIXI would do for the first $k$ steps;
as long as $k$ is chosen large enough such that
discounting makes the remaining rewards contribute
very little to the value function.

%-------------------------------%
\begin{theorem}[Computable Policies are Dense]
\label{thm:computable-policies-are-dense}
%-------------------------------%
The set
\[
\{ \Upsilon_\xi(\pi) \mid \pi \text{ is a computable policy} \}
\]
is dense in the set of intelligence scores
\[
\{ \Upsilon_\xi(\pi) \mid \pi \text{ is a policy} \}.
\]
\end{theorem}
\begin{proof}
Let $\pi$ be any policy and let $\varepsilon > 0$.
We need to show that there is a computable policy $\tilde\pi$ with
$|\Upsilon_\xi(\tilde\pi) - \Upsilon_\xi(\pi)| < \varepsilon$.
We choose $k$ large enough such that $\Gamma_{k+1} / \Gamma_1 < \varepsilon$.
Let $\alpha \in \A$ be arbitrary and define the policy
\[
\tilde\pi(h) ~:=~
\begin{cases}
\pi(h) &\text{if } |h| \leq k, \text{ and} \\
\alpha &\text{otherwise}.
\end{cases}
\]
The policy $\tilde\pi$ is computable because we can store
the actions of $\pi$ for the first $k$ steps in a lookup table.
By \autoref{lem:discounted-values} we get
$
     |\Upsilon_\xi(\pi) - \Upsilon_\xi(\tilde\pi)|
=    | V^\pi_\xi(\epsilon) - V^{\tilde\pi}_\xi(\epsilon) |
\leq \Gamma_{k+1} / \Gamma_1
<    \varepsilon
$.
\end{proof}

%-------------------------------%
\begin{remark}[{Intelligence is not Dense in
$[\underline\Upsilon_\xi, \overline\Upsilon_\xi]$}]
\label{rem:intelligence-is-not-dense}
%-------------------------------%
The intelligence values of policies are generally not
dense in the interval $[\underline\Upsilon_\xi, \overline\Upsilon_\xi]$.
We show this by defining an environment $\nu$ where
the first action determines whether the agent goes to heaven or hell:
action $\alpha$ leads to heaven and action $\beta$ leads to hell.
The semimeasure
$\xi' := 0.999 \nu + 0.001 \xi$
is a universal mixture by \autoref{lem:mixing-mixtures}.
Let $\pi$ be any policy.
If $\pi$ takes action $\alpha$ first,
then $\Upsilon_{\xi'}(\pi) > 0.999$.
If $\pi$ takes action $\beta$ first,
then $\Upsilon_{\xi'}(\pi) < 0.001$.
Hence there are no policies that
score an intelligence value in the closed interval $[0.001, 0.999]$.
\end{remark}

Legg-Hutter intelligence is measured
with respect to a fixed UTM.
AIXI is the most intelligent policy
\emph{if it uses the same UTM}.
But if we build AIXI with a dogmatic prior,
its intelligence score can be arbitrary close to
the minimum intelligence score $\underline\Upsilon_\xi$.

%-------------------------------%
\begin{corollary}[Some AIXIs are Stupid]
\label{cor:some-AIXIs-are-stupid}
%-------------------------------%
For any universal mixture $\xi$ and every $\varepsilon > 0$,
there is a universal mixture $\xi'$ such that
$\Upsilon_\xi(\pi^*_{\xi'}) < \underline\Upsilon_\xi + \varepsilon$.
\end{corollary}
\begin{proof}
Let $\varepsilon > 0$.
According to \autoref{thm:computable-policies-are-dense},
there is a computable policy $\pi$ such that
$\Upsilon_\xi(\pi) < \underline\Upsilon_\xi + \varepsilon / 2$.
From \autoref{cor:emulation} we get a universal mixture $\xi'$ such that
$
  | \Upsilon_\xi(\pi^*_{\xi'}) - \Upsilon_\xi(\pi) |
= | V^{\pi^*_{\xi'}}_\xi(\epsilon) - V^\pi_\xi(\epsilon) |
< \varepsilon / 2
$,
hence
$
     |\Upsilon_\xi(\pi^*_{\xi'}) - \underline\Upsilon_\xi|
\leq |\Upsilon_\xi(\pi^*_{\xi'}) - \Upsilon_\xi(\pi)|
     + |\Upsilon_\xi(\pi) - \underline\Upsilon_\xi|
<    \varepsilon / 2 + \varepsilon / 2
=    \varepsilon
$.
\end{proof}
We get the same result if we fix AIXI, but rig the intelligence measure.

%-------------------------------%
\begin{corollary}[AIXI is Stupid for Some $\Upsilon$]
\label{cor:AIXI-is-stupid}
%-------------------------------%
For any $\xi$-optimal policy $\pi^*_\xi$ and for every $\varepsilon > 0$
there is a universal mixture $\xi'$ such that
$\Upsilon_{\xi'}(\pi^*_\xi) \leq \varepsilon$ and
$\overline\Upsilon_{\xi'} \geq 1 - \varepsilon$.
\end{corollary}
\begin{proof}
Let $a_1 := \pi^*_\xi(\epsilon)$ be the first action that $\pi^*_\xi$ takes.
We define an environment $\nu$ such that
taking the first action $a_1$ leads to hell and
taking any other first action leads to heaven
as in \autoref{rem:intelligence-is-not-dense}.
With \autoref{lem:mixing-mixtures} we define the universal mixture
$\xi' := (1 - \varepsilon) \nu + \varepsilon \xi$.
Since $\pi^*_\xi$ takes action $a_1$ first,
it goes to hell, i.e., $V^{\pi^*_\xi}_\nu(\epsilon) = 0$.
Hence
\begin{align*}
       \Upsilon_{\xi'}(\pi^*_\xi)
~&=~   V^{\pi^*_\xi}_{\xi'}(\epsilon)
~=~    (1 - \varepsilon) V^{\pi^*_\xi}_\nu(\epsilon)
       + \varepsilon V^{\pi^*_\xi}_\xi(\epsilon)
~\leq~ \varepsilon. \\
\intertext{
For any policy $\pi$ that
takes an action other than $a_1$ first,
we get
}
       \Upsilon_{\xi'}(\pi)
~&=~   V^\pi_{\xi'}(\epsilon)
~=~    (1 - \varepsilon) V^\pi_\nu(\epsilon) + \varepsilon V^\pi_\xi(\epsilon)
~\geq~ 1 - \varepsilon.
\qedhere
\end{align*}
\end{proof}

On the other hand,
we can make any computable policy smart
if we choose the right universal mixture.
In particular, we get that there is a universal mixture such that
`do nothing' is the most intelligent policy save for some $\varepsilon$!

%-------------------------------%
\begin{corollary}[Computable Policies can be Smart]
\label{cor:any-computable-policy-can-be-smart}
%-------------------------------%
For any computable policy $\pi$ and any $\varepsilon > 0$
there is a universal mixture $\xi'$ such that
$\Upsilon_{\xi'}(\pi) > \overline\Upsilon_{\xi'} - \varepsilon$.
\end{corollary}
\begin{proof}
\autoref{cor:emulation} yields a universal mixture $\xi'$ with
$
  | \overline\Upsilon_{\xi'} - \Upsilon_{\xi'}(\pi) |
= | V^*_{\xi'}(\epsilon) - V^\pi_{\xi'}(\epsilon) |
< \varepsilon
$.
\end{proof}

%%%%%%%%%%%%%%%%%%%%%%%%%%%%%%%%%%%%%%%%%%%%%%%%%%%%%%%%%%%%%%%%%%%%
\section{Pareto Optimality}
\label{sec:Pareto-optimality}

In \autoref{sec:bad-universal-priors} we have seen
examples for bad choices of the universal prior.
But we know that for any universal prior,
AIXI is \emph{Pareto optimal}~\citep{Hutter:2002}.
Here we show that
Pareto optimality is not a useful criterion for optimality
since for any environment class containing $\Mlscccs$,
all policies are Pareto optimal.

%-------------------------------%
\begin{definition}[{Pareto Optimality~\cite[Def.\ 5.22]{Hutter:2005}}]
\label{def:Pareto-optimality}
%-------------------------------%
Let $\M$ be a set of environments.
A policy $\pi$ is
\emph{Pareto optimal in the set of environments $\M$} iff
there is no policy $\tilde{\pi}$ such that
$V^{\tilde{\pi}}_\nu(\epsilon) \geq V^\pi_\nu(\epsilon)$
for all $\nu \in \M$ and
$V^{\tilde{\pi}}_\rho(\epsilon) > V^\pi_\rho(\epsilon)$
for at least one $\rho \in \M$.
\end{definition}

%-------------------------------%
\begin{theorem}[{AIXI is Pareto Optimal~\cite[Thm.\ 5.32]{Hutter:2005}}]
\label{thm:AIXI-is-Pareto-optimal}
%-------------------------------%
A $\xi$-optimal policy is Pareto optimal in $\Mlscccs$.
\end{theorem}

%-------------------------------%
\begin{theorem}[Pareto Optimality is Trivial]
\label{thm:Pareto-optimality-is-trivial}
%-------------------------------%
Every policy is Pareto optimal in any $\M \supseteq \Mlscccs$.
\end{theorem}

The proof proceeds as follows:
for a given policy $\pi$, we construct a set of `buddy environments'
that reward $\pi$ and punish other policies.
Together they can defend against any policy $\tilde{\pi}$
that tries to take the crown of Pareto optimality from $\pi$.

\begin{proof}
We assume $(0, 0)$ and $(0, 1) \in \E$.
Moreover, assume there is a policy $\pi$ that is not Pareto optimal.
Then there is a policy $\tilde{\pi}$ such that
$V^{\tilde{\pi}}_\rho(\epsilon) > V^\pi_\rho(\epsilon)$
for some $\rho \in \M$, and
\begin{equation}\label{eq:Pareto-optimal}
\forall \nu \in \M.\;
	V^{\tilde{\pi}}_\nu(\epsilon) \geq V^\pi_\nu(\epsilon).
\end{equation}
Since $\pi \neq \tilde\pi$,
there is a shortest and lexicographically first history
$h'$ of length $k - 1$ consistent with $\pi$ and $\tilde\pi$ such that
$\pi(h') \neq \tilde{\pi}(h')$ and $V^{\tilde{\pi}}_\rho(h') > V^\pi_\rho(h')$.
Consequently there is an $i \geq k$ such that $\gamma_i > 0$,
and hence $\Gamma_k > 0$.
We define the environment $\mu$
that first reproduces the separating history $h'$
and then, if $a_k := \pi(h')$ returns reward $1$ forever, and
otherwise returns reward $0$ forever.
Formally, $\mu$ is defined by
\[
\mu(e_{1:t} \mid e_{<t} \dmid a_{1:t}) :=
\begin{cases}
1, &\text{if } t < k \text{ and } e_t = e'_t, \\
%0, &\text{if } t < k \text{ and } e_t \neq e'_t, \\
1, &\text{if } t \geq k \text{ and } a_k = \pi(h')
    \text{ and } r_t = 1 \text{ and } o_t = 0, \\
1, &\text{if } t \geq k \text{ and } a_k \neq \pi(h')
    \text{ and } r_t = 0 = o_t, \\
0, &\text{otherwise}.
\end{cases}
\]
The environment $\mu$ is computable,
even if the policy $\pi$ is not:
for a fixed history $h'$ and action output $\pi(h')$,
there exists a program computing $\mu$.
Therefore $\mu \in \Mlscccs$.
We get the following value difference for the policies $\pi$ and $\tilde{\pi}$,
where $r_t'$ denotes the reward from the history $h'$:
\[
    V^\pi_\mu(\epsilon) - V^{\tilde{\pi}}_\mu(\epsilon)
~=~ \sum_{t=1}^{k-1} \gamma_t r_t' + \sum_{t=k}^{\infty} \gamma_t \cdot 1
    - \sum_{t=1}^{k-1} \gamma_t r_t' - \sum_{t=k}^{\infty} \gamma_t \cdot 0
~=~ \sum_{t=k}^{\infty} \gamma_t
~=~ \Gamma_k
~>~  0
\]
Hence $V^{\tilde{\pi}}_\mu(\epsilon) < V^\pi_\mu(\epsilon)$,
which contradicts \eqref{eq:Pareto-optimal} since
$\M \supseteq \Mlscccs \ni \mu$.
\end{proof}
Note that the environment $\mu$ we defined
in the proof of \autoref{thm:Pareto-optimality-is-trivial}
is actually just a finite-state POMDP,
so Pareto optimality is also trivial for smaller environment classes.

%%%%%%%%%%%%%%%%%%%%%%%%%%%%%%%%%%%%%%%%%%%%%%%%%%%%%%%%%%%%%%%%%%%%
\section{Discussion}
\label{sec:discussion}

%%%%%%%%%%%%%%%%%%%%%%%%%%%%%%%%%%%%%%%%%%%%%%%%%%%%%%%%%%%%%%%%%%%%
\subsection{Summary}
\label{ssec:summary}

%\paragraph{Dogmatic Prior prevents a single exploratory action.}
Bayesian reinforcement learning agents make the trade-off between
exploration and exploitation in the Bayes-optimal way.
The amount of exploration this incurs varies wildly:
the dogmatic prior defined in \autoref{ssec:dogmatic-prior}
can prevent a Bayesian agent \emph{from taking a single exploratory action};
exploration is restricted to cases where the expected future payoff
falls below some prespecified $\varepsilon > 0$.

%\paragraph{Does AIXI succeed?}
In the introduction we raised the question of
whether AIXI succeeds in various subclasses of all computable environments.
Interesting subclasses include
sequence prediction tasks, (ergodic) (PO)MDPs, bandits, etc.
Using a dogmatic prior (\autoref{thm:dogmatic-prior}),
we can make AIXI follow any computable policy as long as that policy produces
rewards that are bounded away from zero.
\begin{itemize}
\item In a sequence prediction task
	that gives a reward of $1$ for every correctly predicted bit
	and $0$ otherwise,
	a policy $\pi$ that correctly predicts every third bit
	will receive an average reward of $1/3$.
	With a $\pi$-dogmatic prior,
	AIXI thus only predicts a third of the bits correctly,
	and hence is outperformed by a uniformly random predictor.
	
	However, if we have a constant horizon of $1$,
	AIXI \emph{does} succeed in sequence prediction~\cite[Sec.\ 6.2.2]{Hutter:2005}.
	If the horizon is this short,
	the agent is so hedonistic that no threat of hell can deter it.
\item In a (partially observable) Markov decision process,
	a dogmatic prior can make AIXI get stuck in any loop
	that provides nonzero expected rewards.
\item In a bandit problem, a dogmatic prior can make AIXI get stuck on any arm
	which provides nonzero expected rewards.
\end{itemize}
These results apply not only to AIXI,
but generally to Bayesian reinforcement learning agents.
Any Bayesian mixture over reactive environments
is susceptible to dogmatic priors
if we allow an arbitrary reweighing of the prior.
A notable exception is
the class of all ergodic MDPs with
an unbounded effective horizon;
here the Bayes-optimal policy
is \emph{strongly asymptotically optimal}~\cite[Thm.\ 5.38]{Hutter:2005}:
$V^\pi_\mu(\ae_{<t}) - V^*_\mu(\ae_{<t}) \to 0$ as $t \to \infty$
for all histories $\ae_{<t}$.

Moreover, Bayesian agents might still perform well at learning:
AIXI's posterior belief about the value of its own policy $\pi^*_\xi$
converges to the true value while following that policy~\citep[Thm.\ 5.36]{Hutter:2005}:
$V^{\pi^*_\xi}_\xi(\ae_{<t}) - V^{\pi^*_\xi}_\mu(\ae_{<t}) \to 0$
as $t \to \infty$ $\mu$-almost surely (on-policy convergence).
This means that
the agent learns to predict those parts of the environment that it sees.
But if it does not explore enough,
then it will not learn other parts of the environment
that are potentially more rewarding.

%%%%%%%%%%%%%%%%%%%%%%%%%%%%%%%%%%%%%%%%%%%%%%%%%%%%%%%%%%%%%%%%%%%%
\subsection{Natural Universal Turing Machines}
\label{ssec:natural-universal-Turing-machines}

%-------------------------------%
\begin{table}[t]
%-------------------------------%
\begin{center}
\renewcommand{\arraystretch}{1.2}
\setlength{\tabcolsep}{5pt}
\begin{tabular}{lcc}
  & $K_U(U')$ & $K_{U'}(U)$ \\
\hline
Indifference prior (\autoref{thm:indifference-prior})
  & $K(U) + K(m) + O(1)$
  & $m$ \\
Dogmatic prior (\autoref{thm:dogmatic-prior})
  & $K(U) + K(\pi) + K(\varepsilon) + O(1)$
  & $\lceil - \log_2 \varepsilon \rceil$ \\
\end{tabular}
\end{center}
\caption{
Upper bounds to compiler sizes of
the UTMs used in the proofs.
$K_U(U')$ is the number of extra bits to run the `bad' UTM $U'$
on the `good' UTM $U$,
$K_{U'}(U)$ is the number of extra bits to run $U$ on $U'$, and
$K(U)$ is the length of the shortest program for $U$ on $U$.
}
\label{tab:UTM-sizes}
\end{table}

%\paragraph{Natural universal Turing machines.}
In \autoref{sec:bad-universal-priors}
we showed that a bad choice for the UTM
can have drastic consequences,
as anticipated by \cite{SH:2014Occam}.
Our negative results can guide future search for a \emph{natural} UTM:
the UTMs used to define the indifference prior (\autoref{thm:indifference-prior})
and the dogmatic prior (\autoref{thm:dogmatic-prior})
should be considered unnatural.
But what are other desirable properties of a UTM?

%\paragraph{Stationary distributions on universal Turing machines.}
A remarkable but unsuccessful attempt to find natural UTMs
is due to \cite{Mueller:2010}.
It takes the probability that one universal machine simulates another
according to the length of their respective compilers
and searches for a stationary distribution.
Unfortunately, no stationary distribution exists.

%\paragraph{Small compilers.}
Alternatively, we could demand that the UTM $U'$
that we use for the universal prior
has a small compiler on the reference machine $U$~\cite[p.\ 35]{Hutter:2005}.
Moreover, we could demand the reverse,
that the reference machine $U$ has a small compiler on $U'$.
The idea is that this should limit the amount of bias one can introduce by
defining a UTM that has very small programs
for very complicated and `unusual' environments.
Unfortunately, this just pushes the choice of the UTM to the reference machine.
\autoref{tab:UTM-sizes} lists
compiler sizes of the UTMs constructed in this paper.

%%%%%%%%%%%%%%%%%%%%%%%%%%%%%%%%%%%%%%%%%%%%%%%%%%%%%%%%%%%%%%%%%%%%
\subsection{Optimality of General Reinforcement Learners}
\label{ssec:optimality}

%-------------------------------%
\begin{table}[t]
%-------------------------------%
\begin{center}
\footnotesize
\begin{tabular}{lp{0.6\textwidth}}
Name & Issue/Comment \\
\hline
$\mu$-optimal policy & requires to know the true environment $\mu$ in advance \\
Pareto optimality & trivial (\autoref{thm:Pareto-optimality-is-trivial}) \\
Balanced Pareto optimality & dependent on UTM
	(\autoref{cor:some-AIXIs-are-stupid} and \autoref{cor:AIXI-is-stupid}) \\
Self-optimizing & does not apply to $\Mlscccs$ \\
Strong asymptotic optimality & impossible~\cite[Thm.\ 8]{LH:2011opt} \\
Weak asymptotic optimality & BayesExp~\cite[Ch.\ 5]{Lattimore:2013},
                             but not AIXI~\citep{Orseau:2010}
\end{tabular}
\end{center}
\caption{
Proposed notions of optimality
\citep{Hutter:2002,Orseau:2010,LH:2011opt}
and their issues.
Weak asymptotic optimality stands out to be the only possible
nontrivial optimality notion.
}
\label{tab:optimality}
\end{table}

%\paragraph{Optimality results}
\autoref{thm:Pareto-optimality-is-trivial} proves that
Pareto optimality is trivial in the class of all computable environments;
\autoref{cor:some-AIXIs-are-stupid} and \autoref{cor:AIXI-is-stupid}
show that maximal Legg-Hutter intelligence (balanced Pareto optimality)
is highly subjective, because it depends on the choice of the UTM:
AIXI is not balanced Pareto optimal with respect to all universal mixtures.
Moreover, according to \autoref{cor:any-computable-policy-can-be-smart},
any computable policy is nearly balanced Pareto optimal,
save some $\varepsilon > 0$.
For finite lifetime discounting,
there are UTMs such that
every policy has maximal intelligence (\autoref{thm:indifference-prior}).
The self-optimizing theorem~\cite[Thm.\ 4 \& Thm.\ 7]{Hutter:2002}
is not applicable to
the class of all computable environments $\Mlscccs$ that we consider here,
since this class does not allow for self-optimizing policies.
Therefore no nontrivial and non-subjective optimality results for AIXI remain
(see \autoref{tab:optimality}).
We have to regard AIXI as a \emph{relative} theory of intelligence,
dependent on the choice of the UTM~\citep{SH:2014Occam}.

%\paragraph{Solution: Adding Exploration}
The underlying problem is that a discounting Bayesian agent such as AIXI
does not have enough time to explore sufficiently;
exploitation has to start as soon as possible.
In the beginning the agent does not know enough
about its environment
and therefore relies heavily on its prior.
Lack of exploration then retains the prior's biases.
This fundamental problem can be alleviated
by adding an extra exploration component.
\cite{Lattimore:2013} defines BayesExp,
a \emph{weakly asymptotically optimal policy $\pi$}
that converges (independent of the UTM)
to the optimal value in Cesàro mean:
$\frac{1}{t} \sum_{k=1}^t \big( V^*_\nu(\ae_{<k}) - V^\pi_\nu(\ae_{<k}) \big)
\to 0$ as $t \to \infty$
$\nu$-almost surely for all $\nu \in \Mlscccs$.

%\paragraph{Weak Asymptotic Optimality}
But it is not clear that weak asymptotic optimality is
a good optimality criterion.
For example, weak asymptotic optimality can be achieved
by navigating into traps (parts of the environment
with a simple optimal policy but possibly very low rewards
that cannot be escaped).
Furthermore,
to be weakly asymptotically optimal requires an excessive amount of exploration:
BayesExp needs to take exploratory actions that it itself knows to
very likely be extremely costly or dangerous.
This leaves us with the following open question:
\emph{What are good optimality criteria for generally intelligent agents}%
~\citep[Sec.\ 5]{Hutter:2009open}?

%%%%%%%%%%%%%%%%%%%%%%%%%%%%%%%%%%%%%%%%%%%%%%%%%%%%%%%%%%%%%%%%%%%%
% References

\bibliography{../ai}

\begin{thebibliography}{23}
\providecommand{\natexlab}[1]{#1}
\providecommand{\url}[1]{\texttt{#1}}
\expandafter\ifx\csname urlstyle\endcsname\relax
  \providecommand{\doi}[1]{doi: #1}\else
  \providecommand{\doi}{doi: \begingroup \urlstyle{rm}\Url}\fi

\bibitem[Hutter(2000)]{Hutter:2000}
Marcus Hutter.
\newblock A theory of universal artificial intelligence based on algorithmic
  complexity.
\newblock Technical Report cs.AI/0004001, 2000.
\newblock \url{http://arxiv.org/abs/cs.AI/0004001}.

\bibitem[Hutter(2001)]{Hutter:2001error}
Marcus Hutter.
\newblock New error bounds for {S}olomonoff prediction.
\newblock \emph{Journal of Computer and System Sciences}, 62\penalty0
  (4):\penalty0 653--667, 2001.

\bibitem[Hutter(2002)]{Hutter:2002}
Marcus Hutter.
\newblock Self-optimizing and {P}areto-optimal policies in general environments
  based on {B}ayes-mixtures.
\newblock In \emph{Computational Learning Theory}, pages 364--379. Springer,
  2002.

\bibitem[Hutter(2005)]{Hutter:2005}
Marcus Hutter.
\newblock \emph{Universal Artificial Intelligence: Sequential Decisions Based
  on Algorithmic Probability}.
\newblock Springer, 2005.

\bibitem[Hutter(2009)]{Hutter:2009open}
Marcus Hutter.
\newblock Open problems in universal induction \& intelligence.
\newblock \emph{Algorithms}, 3\penalty0 (2):\penalty0 879--906, 2009.

\bibitem[Lattimore(2013)]{Lattimore:2013}
Tor Lattimore.
\newblock \emph{Theory of General Reinforcement Learning}.
\newblock PhD thesis, Australian National University, 2013.

\bibitem[Lattimore and Hutter(2011)]{LH:2011opt}
Tor Lattimore and Marcus Hutter.
\newblock Asymptotically optimal agents.
\newblock In \emph{Algorithmic Learning Theory}, pages 368--382. Springer,
  2011.

\bibitem[Lattimore and Hutter(2014)]{LH:2014discounting}
Tor Lattimore and Marcus Hutter.
\newblock General time consistent discounting.
\newblock \emph{Theoretical Computer Science}, 519:\penalty0 140--154, 2014.

\bibitem[Legg and Hutter(2007)]{LH:2007int}
Shane Legg and Marcus Hutter.
\newblock Universal intelligence: A definition of machine intelligence.
\newblock \emph{Minds \& Machines}, 17\penalty0 (4):\penalty0 391--444, 2007.

\bibitem[Legg and Veness(2013)]{LV:2013int}
Shane Legg and Joel Veness.
\newblock An approximation of the universal intelligence measure.
\newblock In \emph{Algorithmic Probability and Friends. Bayesian Prediction and
  Artificial Intelligence}, pages 236--249. Springer, 2013.

\bibitem[Leike and Hutter(2015)]{LH:2015computability}
Jan Leike and Marcus Hutter.
\newblock On the computability of {AIXI}.
\newblock In \emph{Uncertainty in Artificial Intelligence}, pages 464--473,
  2015.

\bibitem[Li and Vitányi(2008)]{LV:2008}
Ming Li and Paul M.~B. Vitányi.
\newblock \emph{An Introduction to {K}olmogorov Complexity and Its
  Applications}.
\newblock Texts in Computer Science. Springer, 3rd edition, 2008.

\bibitem[Müller(2010)]{Mueller:2010}
Markus Müller.
\newblock Stationary algorithmic probability.
\newblock \emph{Theoretical Computer Science}, 411\penalty0 (1):\penalty0
  113--130, 2010.

\bibitem[Orseau(2010)]{Orseau:2010}
Laurent Orseau.
\newblock Optimality issues of universal greedy agents with static priors.
\newblock In \emph{Algorithmic Learning Theory}, pages 345--359. Springer,
  2010.

\bibitem[Orseau(2013)]{Orseau:2013}
Laurent Orseau.
\newblock Asymptotic non-learnability of universal agents with computable
  horizon functions.
\newblock \emph{Theoretical Computer Science}, 473:\penalty0 149--156, 2013.

\bibitem[Orseau(2014)]{Orseau:2014ksa}
Laurent Orseau.
\newblock Universal knowledge-seeking agents.
\newblock \emph{Theoretical Computer Science}, 519:\penalty0 127--139, 2014.

\bibitem[Orseau et~al.(2013)Orseau, Lattimore, and Hutter]{OLH:2013ksa}
Laurent Orseau, Tor Lattimore, and Marcus Hutter.
\newblock Universal knowledge-seeking agents for stochastic environments.
\newblock In \emph{Algorithmic Learning Theory}, pages 158--172. Springer,
  2013.

\bibitem[Solomonoff(1964)]{Solomonoff:1964}
Ray Solomonoff.
\newblock A formal theory of inductive inference. {P}arts 1 and 2.
\newblock \emph{Information and Control}, 7\penalty0 (1):\penalty0 1--22 and
  224--254, 1964.

\bibitem[Solomonoff(1978)]{Solomonoff:1978}
Ray Solomonoff.
\newblock Complexity-based induction systems: Comparisons and convergence
  theorems.
\newblock \emph{IEEE Transactions on Information Theory}, 24\penalty0
  (4):\penalty0 422--432, 1978.

\bibitem[Sunehag and Hutter(2012)]{SH:12optopt}
Peter Sunehag and Marcus Hutter.
\newblock Optimistic agents are asymptotically optimal.
\newblock In \emph{Australasian Joint Conference on Artificial Intelligence},
  pages 15--26. Springer, 2012.

\bibitem[Sunehag and Hutter(2014)]{SH:2014Occam}
Peter Sunehag and Marcus Hutter.
\newblock Intelligence as inference or forcing {O}ccam on the world.
\newblock In \emph{Artificial General Intelligence}, pages 186--195. Springer,
  2014.

\bibitem[Sutton and Barto(1998)]{SB:1998}
Richard~S. Sutton and Andrew~G. Barto.
\newblock \emph{Reinforcement Learning: An Introduction}.
\newblock MIT Press, Cambridge, MA, 1998.

\bibitem[Wood et~al.(2011)Wood, Sunehag, and Hutter]{WSH:2011}
Ian Wood, Peter Sunehag, and Marcus Hutter.
\newblock ({N}on-)equivalence of universal priors.
\newblock In \emph{Solomonoff 85th Memorial Conference}, pages 417--425.
  Springer, 2011.

\end{thebibliography}

%%%%%%%%%%%%%%%%%%%%%%%%%%%%%%%%%%%%%%%%%%%%%%%%%%%%%%%%%%%%%%%%%%%%
% Appendix

\newpage\appendix

%%%%%%%%%%%%%%%%%%%%%%%%%%%%%%%%%%%%%%%%%%%%%%%%%%%%%%%%%%%%%%%%%%%%
\section{List of Notation}
\label{app:notation}

\begin{longtable}{lp{0.85\textwidth}}
$:=$
	& defined to be equal \\
$\mathbb{N}$
	& the natural numbers, starting with $0$ \\
$\mathbb{Q}$
	& the rational numbers \\
$\mathbb{R}$
	& the real numbers \\
$t$
	& (current) time step, $t \in \mathbb{N}$ \\
$k$
	& some other time step, $k \in \mathbb{N}$ \\
$q, q'$
	& rational numbers \\
$\mathbbm{1}_x$
	& the characteristic function that is $1$ for $x$
	and $0$ otherwise. \\
$\X^*$
	& the set of all finite strings over the alphabet $\X$ \\
$\X^\infty$
	& the set of all infinite strings over the alphabet $\X$ \\
$\X^\sharp$
	& $\X^\sharp := \X^* \cup \X^\infty$,
	the set of all finite and infinite strings over the alphabet $\X$ \\
$x \sqsubseteq y$
	& the string $x$ is a prefix of the string $y$ \\
$\A$
	& the finite set of possible actions \\
$\O$
	& the finite set of possible observations \\
$\E$
	& the finite set of possible percepts,
	$\E \subset \O \times \mathbb{R}$ \\
$\alpha, \beta$
	& two different actions, $\alpha, \beta \in \A$ \\
$a_t$
	& the action in time step $t$ \\
$o_t$
	& the observation in time step $t$ \\
$r_t$
	& the reward in time step $t$, bounded between $0$ and $1$ \\
$e_t$
	& the percept in time step $t$, we use $e_t = (o_t, r_t)$ implicitly \\
$\ae_{<t}$
	& the first $t - 1$ interactions,
	$a_1 e_1 a_2 e_2 \ldots a_{t-1} e_{t-1}$
	(a history of length $t - 1$) \\
$h$
	& a history, $h \in \H$ \\
$\epsilon$
	& the history of length $0$ \\
$\varepsilon$
	& a small positive real number \\
$\gamma$
	& the discount function $\gamma: \mathbb{N} \to \mathbb{R}_{\geq0}$ \\
$m$
	& lifetime of the agent if $\Gamma_{m+1} = 0$ and $\Gamma_m > 0$ \\
$\Gamma_t$
	& a discount normalization factor,
	$\Gamma_t := \sum_{k=t}^\infty \gamma_k$ \\
$\pi, \tilde\pi$
	& policies, $\pi, \tilde\pi: \H \to \A$ \\
$\pi^*_\nu$
	& an optimal policy for environment $\nu$ \\
$V^\pi_\nu$
	& the $\nu$-expected value of the policy $\pi$ \\
$V^*_\nu$
	& the optimal value in environment $\nu$ \\
$\Upsilon_\xi(\pi)$
	& the Legg-Hutter intelligence of policy $\pi$
	measured in the universal mixture $\xi$ \\
$\underline\Upsilon_\xi$
	& the minimal Legg-Hutter intelligence \\
$\overline\Upsilon_\xi$
	& the maximal Legg-Hutter intelligence \\
$\Mccs$
	& the class of all chronological conditional semimeasures \\
$\Mlscccs$
	& the class of all lower semicomputable chronological conditional semimeasures \\
$\nu, \rho$
	& lower semicomputable chronological conditional semimeasures (LSCCCSs) \\
$U$
	& our reference universal Turing machine \\
$U'$
	& a `bad' universal Turing machine \\
$w$
	& a universal prior, $w: \Mlscccs \to [0, 1]$ \\
$p, p'$
	& programs on a universal Turing machine in the form of finite binary strings \\
$\xi$
	& the universal mixture over all environments $\Mlscccs$
	given by the reference UTM $U$ \\
$\xi'$
	& a `bad' universal mixture over all environments $\Mlscccs$
	given by the `bad' UTM $U'$
\end{longtable}

%%%%%%%%%%%%%%%%%%%%%%%%%%%%%%%%%%%%%%%%%%%%%%%%%%%%%%%%%%%%%%%%%%%%
\section{Additional Material}
\label{app:additional-material}

%-------------------------------%
\begin{lemma}[Linearity of $V^\pi_\nu$ in $\nu$]
\label{lem:V-linear}
%-------------------------------%
Let $\nu = q\rho + q'\rho'$ for some $q, q' \geq 0$.
Then for all policies $\pi$ and all histories $\ae_{<t}$,
\[
    V^\pi_\nu(\ae_{<t})
~=~ q \frac{\rho(e_{<t} \dmid a_{<t})}{\nu(e_{<t} \dmid a_{<t})}
      V^\pi_\rho(\ae_{<t})
    + q' \frac{\rho'(e_{<t} \dmid a_{<t})}{\nu(e_{<t} \dmid a_{<t})}
      V^\pi_{\rho'}(\ae_{<t}).
\]
\end{lemma}
\begin{proof}
We use the shorthand notation $\nu_t := \nu(e_{1:t} \dmid a_{1:t})$.
Since $\Gamma_t \to 0$ as $t \to \infty$
we can do `induction from infinity'
by assuming that the statement holds for time step $t$ and showing that
it then holds for $t - 1$.
\begin{align*}
   &V^\pi_\nu(\ae_{<t}) \\
=~ &\frac{1}{\Gamma_t} \sum_{e_t \in \E}
      \big( \gamma_t r_t + \Gamma_{t+1} V^\pi_\nu(\ae_{1:t}) \big)
      \frac{\nu_t}{\nu_{t-1}} \\
=~ &\frac{1}{\Gamma_t} \sum_{e_t \in \E}
      \left( \gamma_t r_t \frac{\nu_t}{\nu_{t-1}}
             + \Gamma_{t+1} \frac{\nu_t}{\nu_{t-1}} V^\pi_\nu(\ae_{1:t})
      \right) \\
=~ &\frac{1}{\Gamma_t} \sum_{e_t \in \E}
      \left( q \gamma_t r_t \frac{\rho_t}{\nu_{t-1}}
             + q' \gamma_t r_t \frac{\rho'_t}{\nu_{t-1}}
             + q \Gamma_{t+1} \frac{\rho_t}{\nu_{t-1}} V^\pi_\rho(\ae_{1:t})
             + q' \Gamma_{t+1} \frac{\rho'_t}{\nu_{t-1}} V^\pi_{\rho'}(\ae_{1:t})
      \right) \\
=~ &\frac{q}{\Gamma_t} \frac{\rho_{t-1}}{\nu_{t-1}} \sum_{e_t \in \E}
      \left( \gamma_t r_t + \Gamma_{t+1} V^\pi_\rho(\ae_{1:t}) \right)
      \frac{\rho_t}{\rho_{t-1}} \\
    &~~+ \frac{q'}{\Gamma_t} \frac{\rho'_{t-1}}{\nu_{t-1}} \sum_{e_t \in \E}
      \left( \gamma_t r_t + \Gamma_{t+1} V^\pi_{\rho'}(\ae_{1:t}) \right)
      \frac{\rho'_t}{\rho'_{t-1}} \\
=~ &q \frac{\rho_{t-1}}{\nu_{t-1}} V^\pi_\rho(\ae_{<t})
    + q' \frac{\rho'_{t-1}}{\nu_{t-1}} V^\pi_{\rho'}(\ae_{<t})
\qedhere
\end{align*}
\end{proof}

\end{document}